\documentclass{jocg}

\usepackage{graphicx}
\usepackage{amsmath, amssymb, amsthm}

\newtheorem{thm}{Theorem}
\newtheorem{cor}[thm]{Corollary}
\newtheorem{lem}[thm]{Lemma}
\newtheorem{defn}{Definition}
\newtheorem{proposition}{Proposition}

\newtheorem{obs}{Observation}

\newcommand{\0}{\mathbf{0}}
\newcommand{\R}{\mathbb{R}}

\newcommand\abs[1]{|#1|}
\newcommand\norm[1]{|\hskip-.2ex|#1|\hskip-.2ex|}
\newcommand\SetOf[2]{\left\{#1\,\vphantom{#2}\right.\left|\vphantom{#1}\,#2\right\}}
\newcommand{\sign}{\mathop{{\rm sign}}}
\newcommand{\conv}{\mathop{{\rm conv}}}
\newcommand{\minimize}{\mathop{{\rm minimize}}}
\newcommand{\subjto}{\mathop{{\rm subject~to}}}
\newcommand{\interior}{\mathop{\rm int}}
\newcommand{\Gol}{\mathrm{Gol}}
\newcommand{\lft}{\mathop{{\rm left}}}
\newcommand{\rgt}{\mathop{{\rm right}}}

\newcommand{\xx}{\mathbf{x}}
\newcommand{\bb}{\mathbf{b}}
\newcommand{\cc}{\mathbf{c}}
\newcommand{\Lam}{\mathbf{\Lambda}}

\renewcommand{\aa}{\mathbf{a}}
\newcommand{\uu}{\mathbf{u}}
\newcommand{\vv}{\mathbf{v}}
\newcommand{\ww}{\mathbf{w}}
\newcommand{\Stretch}{L}
\renewcommand{\stretch}{\ell}
\newcommand{\pp}{\mathbf{p}}
\newcommand{\ppp}{{\pp_{\sigma}^{(\stretch)}}}
\newcommand{\qq}{\mathbf{q}}
\newcommand{\qqq}{{\qq_{\sigma}}}
\renewcommand{\SS}{{\cal S}}
\newcommand{\LL}{{\cal L}}
\newcommand{\PP}{{\cal P}}
\newcommand{\FF}{{\cal F}}
\newcommand{\VV}{{\cal V}}

\begin{document}

\title{%
  \MakeUppercase{An Exponential Lower Bound on the Complexity of Regularization Paths}%
}

\author{%
  Bernd~G\"artner,%
  \thanks{\affil{Institute of Theoretical Computer Science, ETH Zurich, Switzerland}, 
          \email{gaertner@inf.ethz.ch}}\,
  Martin Jaggi,%
  \thanks{\affil{CMAP, {\'E}cole Polytechnique, Palaiseau, France},
          \email{jaggi@cmap.polytechnique.fr}}\,
  and Cl\'ement Maria%
  \thanks{\affil{INRIA Sophia Antipolis-M{\'e}diterran{\'e}e, France}, 
          \email{clement.maria@inria.fr}}
}

\maketitle

\begin{abstract}
  For a variety of regularized optimization problems in machine learning, 
  algorithms computing the entire solution path have been developed recently. 
  Most of these methods are quadratic programs that are parameterized by
  a single parameter, as for example the Support Vector Machine (SVM).
  Solution path algorithms do not only compute the solution for one particular value
  of the regularization parameter but the entire path of solutions,
  making the selection of an optimal parameter much easier. 
  
  It has been assumed that these piecewise linear solution paths have only
  linear complexity, i.e.\ linearly many bends. We prove that for the
  support vector machine this complexity can be exponential in
  the number of training points in the worst case. More strongly, we construct 
  a single instance of $n$ input points in $d$ dimensions for an SVM
  such that at least $\Theta(2^{n/2}) = \Theta(2^{d})$ many distinct subsets of 
  support vectors %
  occur as the regularization parameter changes.
\end{abstract}

\section{Introduction}

Regularization methods such as support vector machines (SVM) and 
related kernel methods %
have become
very successful standard tools in many optimization, classification and
regression tasks in a variety of areas, for example signal
processing, statistics, biology, computer vision and computer graphics
 as well as data mining.

These regularization methods have in common that they are convex,
usually quadratic, optimization problems containing a special
parameter in their objective function, called the regularization
parameter, representing the tradeoff between two optimization objectives.
In machine learning the two terms are usually the model complexity
(regularization term) and the accuracy on the training data (loss
term), or in other words the tradeoff between a good generalization
performance and over-fitting.

Such parameterized quadratic programming problems have been studied
extensively in both optimization and machine learning, resulting in many
algorithms that are able to not only compute solutions at a single value of 
the parameter, but along the whole solution path as the parameter varies.
For many variants, it is known that the solution paths are piecewise linear
functions in the parameter, however, the complexity of these paths
remained unknown.

Here we prove that the complexity of the solution path for SVMs, which 
are simple instances of parameterized quadratic programs, is indeed 
exponential in the worst case. Furthermore, our example shows that 
exponentially many distinct subsets of support vectors of the optimal 
solution occur as the regularization parameter changes. Here 
the ``exponentially many'' is valid both in terms of the number of input 
points and also in the dimension of the space containing the points.

\subsection{Parameterized Quadratic Programming}
In this paper, we consider \emph{parameterized} quadratic programs of the form 
\begin{equation}
\label{eq:pQP}
\begin{array}{llll}
\mbox{\bf QP$(\mu)$} &\minimize_{\xx}   & \xx^TQ(\mu)\xx + \cc(\mu)^T\xx\\
        &\subjto & A(\mu)\xx \geq \bb(\mu)\\
        &                  & \xx \geq 0,
\end{array}
\end{equation}
where we suppose that $A:\R\rightarrow\R^{m\times n}$, 
$\bb:\R\rightarrow\R^m$ and $Q:\R\rightarrow\R^{n\times n}$, $\cc:\R\rightarrow\R^n$
are functions that describe how the objective function (given by $Q$ and $\cc$)
and the constraints (given by $A$ and $\bb$) vary with some real parameter $\mu$.
Here we assume that $Q$ is always a symmetric positive semi-definite
matrix, as for example a Gram (or kernel) matrix. 

Methods that fit exactly into the above form~(\ref{eq:pQP}) include the $C$- and 
$\nu$-SVM versions with both $\ell_{1}$- and
$\ell_{2}$-loss \cite{Burges:1998hg,Chen:2005p2552}, support vector
regression \cite{Smola:2004ba}, the Lasso for regression and classification 
\cite{Tibshirani:1996wb},
the one-class SVM \cite{implicitKernelSurface}, multiple kernel learning with 2 
kernels \cite{Giesen:2010fx}, $\ell_1$-regularized
least squares \cite{SeungJeanKim:2007er}, least angle regression (LARS) 
\cite{Efron:2004tz}, and also the basis pursuit
denoising problem in compressed sensing \cite{Figueiredo:2007hz}. Parametric quadratic programs are not limited to machine learning, but are also 
very important in control theory (e.g. model predictive control, \cite{Garcia:1989hd}),
 and also occur in geometry as for example polytope distance and 
smallest enclosing ball of moving points \cite{Giesen:2010fx}, and also in many 
finance applications such as mean-variance portfolio selection 
\cite{Markowitz:1952tg} as well as other instances of multi-variate optimization.

The task of solving such a problem for all possible values of the
parameter $\mu$ is called \emph{parametric quadratic programming}.
What we want to compute is a \emph{solution path}, an explicit function
$\xx^*:\R\rightarrow\R^n$ that describes the solution as a function of
the parameter $\mu$.
It is well known that if $\cc$ and $\bb$ are linear
functions in $\mu$, and the matrices $Q$ and $A$ are fixed (do not depend on 
$\mu$), then the solution $\xx^*$ is \emph{piecewise linear} in the parameter 
$\mu$, see for example \cite{Ritter:1962ui}.

We observe that the majority of the above mentioned applications of~(\ref{eq:pQP}) 
are indeed of the special form that only $\cc$ and $\bb$ depend linearly 
on $\mu$, and therefore result in piecewise linear solution paths. This in particular 
holds for the most prominent application in machine learning, the
$\ell_1$-loss SVM, see e.g. \cite{Hastie:2004uj,Rosset:2007fy}. On the 
other hand the $\ell_2$-loss SVM is probably the easiest example where the matrix 
$Q$ is parameterized, while $\cc$ and $\bb$ are fixed there \cite[Equation (13)]
{CVM}.

\subsection{Complexity of Solution Paths}
There are two interesting measures of complexity for the solution paths in the 
parameter~$\mu$ as defined above:
First one can consider the number of pieces or bends in the %
solution path. Here a \emph{bend} is a parameter value $\mu$ at which the 
solution path ``turns'', i.e.\ is not differentiable. Alternatively, one is interested 
in the number of distinct subsets of support vectors that appear as the 
parameter changes. Here a support vector corresponds to a strictly 
non-zero coordinate of the solution to the dual of the quadratic 
program~(\ref{eq:pQP}).

Based on empirical observations, \cite{Hastie:2004uj}
conjectured that the complexity of the solution path of the two-class
SVM, i.e., the number of bends and number of distinct support vectors, is 
linear in the number of training points. 
This empirical conjecture was repeatedly stated for related methods 
in \cite{Hastie:2004uj,Gunter:2005tk,Malioutov:2005cp,Bach:2006uv,Wang:2006cx,Rosset:2007fy,Wang:2007vc,Wang:2007uq,Wang:2008tj,Ong:2010fu,Gu:2012il}.

Here we disprove the conjecture by showing that the complexity in the
SVM case can indeed be exponential in the number of training
points. Our natural construction of $n = 2d+2$ many input points for the SVM 
program~(\ref{eq:pQP}) in $d$-dimensional space has two main interesting 
properties: First we have that all $\Theta(2^{d}) = \Theta(2^{n/2})$ subsets of size $d$ of support vectors do indeed occur as the (regularization) parameter $\mu$ changes. 
Furthermore, the number of bends in the solution path is $\Theta(2^{d}) = \Theta(2^{n/2})$. 
Here the $\textbf{O}$-notation hides just a constant of $\frac{1}{4}$ or $\frac{1}{8}$ 
respectively.

Our construction therefore proves exponential complexity of the solution paths to 
parameterized quadratic programs, even in the most simple case when only the 
linear part~$\cc(\mu)$ of the objective of a quadratic program~(\ref{eq:pQP}) 
depends linearly on the parameter.

To avoid confusion: our construction does not just show that some
particular algorithm needs exponentially many steps to compute the
solution path, but indeed shows that \emph{any} algorithm reporting
the solution path will need exponential time, because the path in our
example is unique and has exponentially many bends. For a brief
overview on existing solution path algorithms see the following 
Section~\ref{sec:solPathAlgos}.

Conceptually, our construction is motivated by the fact that the standard SVM is 
equivalent to the geometric problem of finding the closest distance between two 
polytopes. In this geometric framework, we employ the \emph{Goldfarb cube}, which 
originally served to prove that the \emph{simplex algorithm} for linear 
programming needs an exponential number of steps under some pivot rule \cite{goldfarbExpSimplex}.
We will formally and algebraically define our instance of 
the program~(\ref{eq:pQP}), and we formally prove optimality of the 
constructed solutions by means of the standard KKT conditions. 
This also implies that our construction could probably be modified to 
give a lower bound complexity for other instances of parameterized quadratic 
programs~(\ref{eq:pQP}), not restricted to SVMs.
Continuing this line of research, \cite{Mairal:2012wu} has recently constructed an example of exponential path complexity for Lasso regression, by using a different (non-geometric) proof technique.

\subsection{Solution Path Algorithms}\label{sec:solPathAlgos}

Solution path algorithms and related homotopy methods have a long history, in 
particular in the optimization community \cite{Ritter:1962ui,Bank:1983va,Ritter:1984wz,Osborne:1992fh} 
and in control theory (e.g. model predictive control, \cite{Garcia:1989hd,Hassibi:1999hs}).
In particular, algorithms to compute the entire solution path for parameterized 
quadratic programs~(\ref{eq:pQP}) were already proposed by 
\cite{Bank:1983va,Ritter:1984wz}; \cite[Chapter 5]{Murty:1988wc} and~\cite{Gartner:2009vv}.

More recently these methods had an independent revival in machine learning, in 
particular for computing exact solution paths in the context of support vector 
machines and related problems \cite{Hastie:2004uj,Rosset:2007fy,Wu:2008tu,Gartner:2009vv}, 
and also regression techniques such as $\ell_1$-regularized least squares \cite{Osborne:2000kg,Malioutov:2005cp,Mairal:2012wu}. Similar methods were also applied by 
\cite{Efron:2004tz,Gunter:2005tk,Lee:2006vw,Wang:2006cx,Bach:2006uv,Wang:2006tn,Loosli:2007wf,GyeminLee:2007ha,Wang:2008tj} 
to special cases of quadratic programs, in particular cases where the solution path is piecewise linear.

In machine learning, a solution path algorithm for the special 
case of the $C$-SVM has been proposed by \cite{Hastie:2004uj}. %
\cite{Efron:2004tz} gave such an algorithm for the Lasso, and later \cite{Loosli:2007wf} 
and~\cite{GyeminLee:2007ha} proposed solution path algorithms for $\nu$-SVM 
and one-class SVM respectively. \cite{Lee:2006vw} do the same for multi-class 
SVMs, and \cite{Wang:2006cx} for the Laplacian SVM. Also for the case of cost 
asymmetric SVMs (where each point class has a separate regularization parameter), 
\cite{Bach:2006uv} have computed the solution path by the same methods. 
Support vector regression (SVR) is interesting as its underlying quadratic program 
depends on two parameters, a regularization parameter (for which the solution 
path was tracked by \cite{Gunter:2005tk,Wang:2006tn,Loosli:2007wf}) 
and a tube-width parameter (for which \cite{Wang:2008tj} obtained a solution 
path algorithm). 

However, the above mentioned specialized methods have the disadvantages that 
they are very specific to each individual problem, and they usually require the principal minors of the matrix $Q$ to be invertible, which is not always realistic when dealing with large numerical data~\cite[Section 5.2]{Hastie:2004uj}.
Later \cite{Wu:2008tu} again pointed out that the SVM path problem is indeed only a specific instance of our general parametric quadratic programming problem~(\ref{eq:pQP}), for which generic path optimization algorithms already exist, see e.g. \cite{Ritter:1984wz,Murty:1988wc} and \cite{Gartner:2009vv}. Also, these methods are valid for arbitrary positive semi-definite matrices $Q$. 
The issue of non-invertible sub-matrices was also addressed in \cite{Gartner:2009vv,Ong:2010fu}.

More recently, \cite{Giesen:2010fx,Jaggi:2011ux,Giesen:2012uk} have proposed to study \emph{approximate} solution paths (with some continuous guarantee, e.g. on the duality gap) instead of the \emph{exact} solution paths of such optimization problems.

\subsection{Relation to Results in the Theory of Linear Programming}
We would like to point out that Goldfarb's original cube construction~
\cite{goldfarbExpSimplex,Goldfarb:1994tf} can already be interpreted as an
exponential lower bound on solution path complexity (not of support 
vector machines, though).

In fact, in the theory of linear programming, it is the Gass-Saaty~\cite{Gass:1955hg} or 
\emph{shadow vertex}~\cite{Borgwardt:1987uj} pivot rule under which the 
simplex method needs exponentially many steps on the Goldfarb cube. 
This rule was originally conceived by Gass and Saaty to solve the 
\emph{parametric linear programming problem} in which the objective 
function depends linearly on a real parameter~$\lambda$, and the goal is to 
compute optimal solutions under all possible parameter 
values~\cite{Gass:1955hg}.

Gass and Saaty have described a method to maintain an optimal solution
as $\lambda$ varies from $-\infty$ to $\infty$, which is a solution path.
Their method can in particular be used to compute an optimal solution
to a non-parametric linear program, given some initial solution. This
is the setting of the shadow vertex pivot rule~\cite{Borgwardt:1987uj}.

Goldfarb's worst case result~\cite{goldfarbExpSimplex,Goldfarb:1994tf} 
can then be rephrased as follows: there exists a family of parametric linear 
programs which have exponentially (in the number of variables) many 
different optimal solutions as the parameter $\lambda$ varies 
between~$-\infty$ and~$0$.

Our contribution is to adapt this result to support vector machines, but 
there are some obstacles to overcome. 
First, the nature of the parameterization (i.e. the regularization) of the 
standard two-class SVM is quite different from Goldfarb's parametric linear 
programs. Secondly, while Goldfarb's solution path is discontinuous (it jumps 
from one optimal solution to the next), we need to provide a continuous 
path for the SVM with a unique solution for every parameter value. 
Our approach is to dualize Goldfarb's contruction, and carefully transform it 
into a standard regularized two-class SVM instance, such that Goldfarb's 
linear objective function turns into a quadratic one with similar geometry.

\section{Support Vector Machines}
The support vector machine (SVM) is a well studied standard tool for classification problems, and is among the most widely applied methods from machine learning. In this paper we will discuss SVMs with a standard $\ell_1$-loss term.
The primal $\nu$-SVM problem \cite{Chen:2005p2552} is given by the following
parameterized quadratic program (the equivalent $C$-SVM is of very similar form):
\begin{equation}
\label{eq:svm}
\begin{array}{llll}
&\minimize_{\ww,\rho,b,\mathbf{\xi}} &\frac{1}{2}\norm{\ww}^2 - \nu \rho + \frac{1}{n} \sum_{i=1}^n \xi_i\\
&\subjto & y_i ( \ww^T \pp_i + b ) \ge \rho - \xi_i \\
&                  & \xi_i \geq 0 ~\forall i \\
&                  & \rho \ge 0.
\end{array}
\end{equation}
Here $y_i\in\{\pm 1\}$ is the class label of data-point $\pp_i \in \R^d$ and $\nu$ is
the regularization parameter. 

\subsection{Geometric Interpretation of the Two-Class SVM}\label{sec:geomsvm}
The dual of the $\nu$-SVM, for $\mu := \frac{2}{n \nu}$, is the following
quadratic program, parameterized by a real number $\mu$. Observe that the
regularization parameter has now moved from the objective function to the
constraints:
\begin{equation}
\label{eq:dsvm}
\begin{array}{llll}
&\minimize_\mathbf{\alpha}   & \sum_{i,j} \alpha_i \alpha_j y_i y_j \pp_i^T \pp_j\\
&\subjto & \sum_{i: y_i = +1} \alpha_i = 1 \\
&                  & \sum_{i: y_i = -1} \alpha_i = 1 \\
&                  & 0 \le \alpha_i \le \mu
\end{array}
\end{equation}
Given a solution to this problem, those vectors $\pp_i$ appearing with a non-zero coefficient~$\alpha_i$ are called the \emph{support vectors}.
Formulation (\ref{eq:dsvm}) is equivalent to the polytope distance problem between the reduced convex
hulls of the two classes of data-points in $\R^d$, or formally 
\begin{equation}\label{eq:redPolyDist}
\begin{array}{ll}
\minimize_{\pp,\qq} & \|\pp-\qq\|^2 \\
\subjto & \pp\in\conv_\mu \left(\SetOf{\pp_{i}}{y_i = +1} \right) \\
                  & \qq\in\conv_\mu \left(\SetOf{\pp_{i}}{y_i = -1} \right).
\end{array}                 
\end{equation}
where for any finite point set $\PP \subset \R^d$,  the \emph{reduced convex hull} of $\PP$ is defined
as
\vspace{-0.2\baselineskip}
\[
\textstyle\conv_\mu(\PP) := \displaystyle\SetOf{\sum_{p \in \PP} \alpha_p p}{0
\le
  \alpha_p \le \mu%
, \ \sum_{p \in \PP} \alpha_p = 1} ,
\]
for a given real parameter $\mu$, $\frac{1}{\abs{\PP}} \le \mu \le 1$. Note that
$\conv_{\mu}(\PP)$ $\subseteq \conv_{\mu'}(\PP)$ $\subseteq \conv(\PP)$ for 
$\mu \leq \mu' \leq 1$.

This geometric interpretation for the $\nu$-SVM formulation~(\ref{eq:svm}) was
originally discovered by \cite{Crisp:2000tz}. Here we can also directly see
the equivalence, if in the formulation~(\ref{eq:dsvm}), we rewrite the objective
function as 
\begin{equation}
\sum_{i,j} \alpha_i \alpha_j y_i y_j \pp_i^T \pp_j 
= \Big\| \sum_{i} \alpha_i y_i \pp_i \Big\|^2
= \Big\| \sum_{\substack{i\,:\\y_i=1}} \alpha_i \pp_i - \sum_{\substack{j\,:\\y_j=-1}} \alpha_j \pp_j
\Big\|^2.
\end{equation}

Note that also the slightly more commonly used $C$-SVM variant is equivalent
to the exactly same geometric distance problem~(\ref{eq:redPolyDist}), as it was
shown in \cite{Bennett:2000vj}. The monotone correspondence of the two 
regularization parameters --- the $C$ and the more geometric parameter $\mu$ --- was explained
in more detail by \cite{Chang:2001wi}. Therefore, our following lower bound
constructions for the solution path complexity will hold for both the $\nu$-SVM
and the $C$-SVM case. For more literature on the topic of reduced convex hulls
and also their role in SVM optimization we refer to \cite{Bern:2001to,Goodrich:2009ir}.

\section{A First Example in Two Dimensions}\label{sec:2dimexmpl}

As a first motivating example, we will construct two simple point classes
in the plane for a two-class SVM with $\ell_1$-loss, such that the solution 
path in the regularization parameter will have complexity at least 
$2(\max(n_{+}, n_{-})-3)$, where 
$n_{+}$ and $n_{-}$ are the sizes of the two point classes.
\cite{Hastie:2004uj}, who also observed that the SVM solution path 
is a piecewise linear function in the regularization parameter, empirically 
suggested that the number of bends in the solution path is roughly 
$k \min(n_{+}, n_{-})$, where $k$ is some number in the range between~$4$ and $6$. 

For our construction, we align a large number $n_{+}$ of points of the one class
on a circle segment, and align the other class of just two vertices
below it, as depicted in Figure~\ref{fig:2dimExample}.

\begin{figure*}[ht]
\vskip 0.5em
\begin{center}
  \centerline{\includegraphics[width=0.95\textwidth]{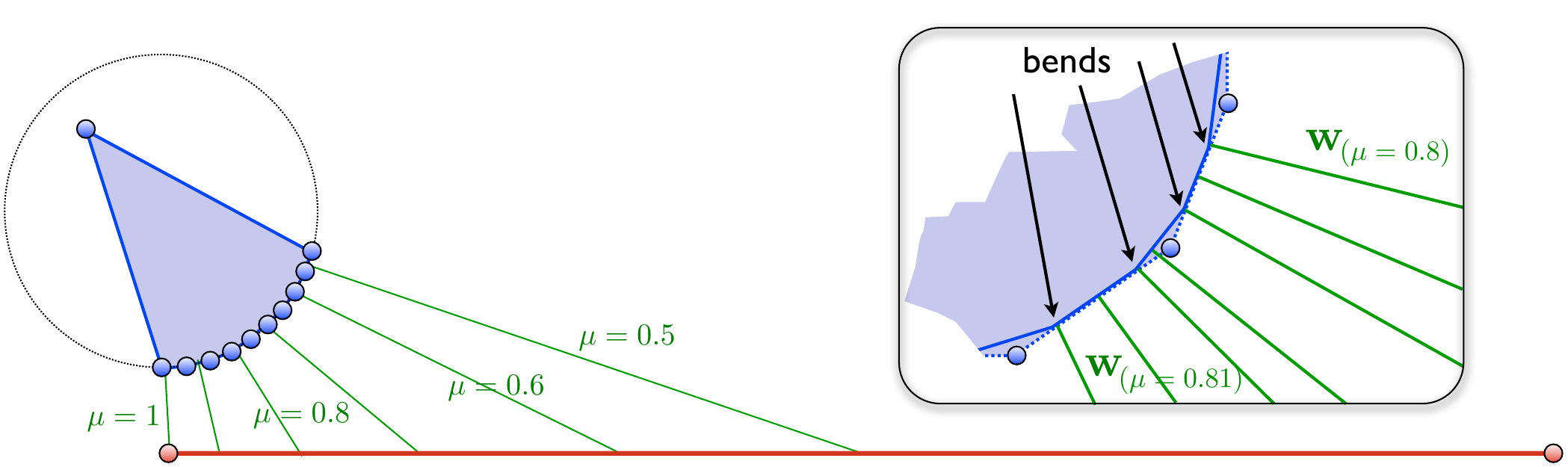}}
\caption{Two dimensional example of an SVM path with at least
  $\max(n_{+}, n_{-})$ many bends. The green lines indicate the
  optimal solutions to the polytope distance problem~(\ref{eq:redPolyDist}), or equivalently the SVM formulations~(\ref{eq:svm}) and~(\ref{eq:dsvm}), for the indicated
  parameter value of $\mu$.}
\label{fig:2dimExample}
\end{center}
\vskip -0.8em
\end{figure*} 

As $\mu$ decreases from $1$ down to $\frac{1}{2}$, the ``left'' end of
the optimal distance vector, which is a multiple of the optimal
$\ww(\mu)$, walks through nearly all of the boundary faces of the
blue class. More precisely, the number of bends in the path of the 
optimal $\ww(\mu)$, for $1 > \mu > \frac{1}{2}$, is at least twice the 
number of ``inner'' blue vertices, which is what we claimed above.

The above argument is not a formal proof, but it 
gives the main idea that will guide us in the high-dimensional 
construction. Going to 
higher dimensions will surprisingly not only allow us to prove
a path complexity lower bound that is linear in the number of input points
$n = n_+ + n_-$, but even exponential in $n$ and also the dimension
$d$ of the space containing the points.

\section{The High-Dimensional Case}\label{sec:highdim} 
The idea is to spice up the two-dimensional example: we will construct
two classes of $n_+=2d$ and $n_-=2$ points, respectively. The point
sets will be in $\R^d$, but the construction ensures that for all
relevant values of the parameter $\mu$, the two points of optimal
distance are very close to the two-dimensional plane 
\begin{equation}\label{eq:S}
\SS := \{\xx\in\R^d: x_1=\ldots=x_{d-2}=0\}.
\end{equation}
The crucial feature of the construction is that the convex
hull of the $n_+$ points intersects~$\SS$ in a convex polygon with
$2^{d}=2^{n_+/2}$ vertices and edges. Moreover, we ``walk through'' a
constant fraction of them while changing the parameter $\mu$. We thus
mimic the process depicted in Figure~\ref{fig:2dimExample},
except that the number of relevant bends is now exponential in $n_+$.

Our main technical tool is the well-known \emph{Goldfarb cube}, a
slightly deformed $d$-dimensional cube with $2d$ facets and $2^d$
vertices \cite{Amenta:1996uo}. Its distinctive property is that all
$2^d$ vertices are visible in the projection of the cube to $\SS$.

Taking the geometric dual of the Goldfarb cube (to be defined below),
we obtain a $d$-dimensional polytope with $2d$ vertices and $2^d$
facets, all of which intersect our two-dimensional plane $\SS$. The $2d$
vertices of the dual Goldfarb cube then form our first point class,
after applying a linear ``stretching transform'' that keeps our walk
close to $\SS$.

\subsection{Polytope Basics}\label{sec:polbasics}
Let us review some basic facts of polytope theory. For proofs,
we refer to Ziegler's standard textbook \cite{zieglerPolytopes}.

Every polytope can be defined in two ways: either as the convex hull
of a finite set of points, or as the bounded solution set of finitely
many linear inequalities. For a given polytope $\PP$, an inequality
$\aa^T\xx\leq b$ is called \emph{face-defining} if $\aa^T\xx\leq b$
for all $\xx\in \PP$ and $\aa^T\xx=b$ for some $\xx\in \PP$.  The set
$\FF=\{\xx\in \PP:\aa^T\xx=b\}$ is called the \emph{face} of $\PP$
defined by the inequality. If $\PP$ has the origin in its interior, it
suffices to consider inequalities of the form $\aa^T\xx\leq 1$. Faces
of dimension $0$ are \emph{vertices}, and faces of dimension $d-1$ are
called \emph{facets}. If $\PP$ is full-dimensional, every vertex is
the intersection of $d$ facets.

Every polytope is the convex hull of its vertices. More generally,
every face $\FF$ is the convex hull of the vertices contained in
$\FF$; in particular $\FF$ is itself a polytope. This is implied by
the following stronger property.

\begin{lem}\label{lem:F} Let $\PP=\conv(\VV)\subseteq\R^d$ be a polytope with 
  vertex set $\VV$, and let $\FF$ be a face of~$\PP$. For every
  point $\pp\in\PP$ and every convex combination
\begin{equation}\label{eq:p_conv}
\pp = \sum_{\vv\in\VV}\alpha_{\vv}\vv, \quad \sum_{\vv\in\VV}\alpha_{\vv}=1,
\quad \alpha_{\vv}\geq 0 ~\forall \vv\in\VV,
\end{equation}
the following two statements are equivalent. 
\begin{itemize}
\item[(i)] $\alpha_{\vv}=0$ for all $\vv\notin\FF$.
\item[(ii)] $\pp\in\FF$.
\end{itemize}
\end{lem}

\begin{proof}
  Let $\aa^T\xx\leq b$ be some inequality that defines $\FF$. If (i)
  holds, then (\ref{eq:p_conv}) yields
\[
\aa^T\pp = \sum_{\vv\in\VV\cap \FF}\alpha_{\vv}\underbrace{\aa^T\vv}_{=b}
= b,
\]
hence $\pp\in\FF$. For the other direction, let $\pp\in\FF$. We get
\[
b =\aa^T\pp = \sum_{\vv\in\VV}\alpha_{\vv}\underbrace{\aa^T\vv}_{\leq b}
\leq \sum_{\vv\in\VV}\alpha_{\vv}b = b,  
\]
where the inequality uses $\alpha_{\vv}\geq 0 ~\forall \vv\in\VV$. It
follows that the inequality is actually an equality, but this is
possible only if $\alpha_{\vv}=0$ whenever $\aa^T\vv<b\Leftrightarrow
\vv\notin\FF$.
\end{proof}

\subsection{The Goldfarb Cube}\label{sec:goldfarb}
The $d$-dimensional Goldfarb cube is a slightly deformed variant of
the cube $[-1,1]^d\subseteq \R^d$. More precisely, it is a polytope
given as the solution set of the following $2d$ linear inequalities.

\begin{defn}\label{def:goldfarb}
  For fixed $\epsilon$ and $\gamma$ such that
  $0<4\gamma<\epsilon<\frac{1}{2}$, the Goldfarb cube $\Gol_d$ is the
  set of points $\xx=(x_1,\ldots,x_d)^T\in\R^d$ satisfying the $2d$
  linear inequalities
\begin{equation}
\label{eq:goldfarb}
\begin{array}{rcccl}
-z_1 & \leq    &  x_1     & \leq & z_1 := 1,               \\
-z_2   & \leq    &  x_2      & \leq & z_2 := 1 - \epsilon - \epsilon x_1, \\
-z_k & \leq & x_{k} & \leq & z_k := 1 - \epsilon + \epsilon\gamma - \epsilon (x_{k-1} - \gamma x_{k-2}), \quad 3\leq k \leq d.
\end{array}
\end{equation}
\end{defn}

We note that the ``standard'' Goldfarb cube as 
in \cite{Amenta:1996uo} is defined differently but can be
obtained from our variant by translation and scaling: under the
coordinate transformation $x_k = 2x'_k - 1$, (\ref{eq:goldfarb}) is
equivalent to Amenta \& Ziegler's Goldfarb cube inequalities \cite{Amenta:1996uo}. The
Goldfarb cube was originally constructed to get a linear program on
which the \emph{simplex algorithm} with the \emph{shadow vertex} pivot
rule needs an exponential number of steps to find the optimal solution
\cite{goldfarbExpSimplex}.

In the following, we state some important properties of the Goldfarb
cube; proofs can be found in \cite{Amenta:1996uo}.

$\Gol_d$ is a full-dimensional polytope with $2d$ facets and the
origin in its interior (this actually holds for all $\epsilon<1$). For
each $k=1,\ldots,d$, the two inequalities $-z_k\leq x_k \leq z_k$ of~(\ref{eq:goldfarb}) define two disjoint ``opposite'' facets. A vertex
is therefore the intersection of exactly $d$ facets, one from each
pair of opposite facets. In fact, every such choice of $d$ facets
yields a distinct vertex which means that there are $2^d$ vertices
that can be indexed by the set $\{-1,1\}^d$. An index vector
$\sigma\in \{-1,1\}^d$ tells us for each pair $-z_k \leq x_k \leq z_k$
of inequalities whether the left one is tight at the vertex
($\sigma_k=-1$), or the right one ($\sigma_k=1$). We can therefore
easily compute the vertices.

\begin{lem}\label{def:goldfarbv} Let $\sigma\in\{-1,1\}^d$. The vector
$\xx=(x_1,\ldots,x_d)^T$ given by 
\begin{equation}
\label{eq:gfvertices}
\begin{array}{rcl}
  x_1 &=& \sigma_1,\\ 
  x_2 &=& \sigma_2(1 - \epsilon - \epsilon x_1),\\
  x_k &=& \sigma_k(1 - \epsilon + \epsilon\gamma - \epsilon (x_{k-1} - \gamma x_{k-2})), 
  \quad k=3,\ldots,d,
\end{array}
\end{equation}
is a vertex of $\Gol_d$ and will be denoted by $\vv_{\sigma}$.
\end{lem}

\begin{cor}\label{cor:gf12}
  Fix $\sigma\in\{-1,1\}^d$ and consider the vertex
  $\vv_{\sigma}=(v_{\sigma,1},\ldots,v_{\sigma,d})^T$.  Then
\[\sign(v_{\sigma,k}) = \sigma_k, \quad 1\leq k\leq d.\]
\end{cor}

\begin{proof}
  Since all the $\vv_{\sigma}$'s are distinct,
  (\ref{eq:gfvertices}) shows that we must in particular have
  $v_{\sigma,k}\neq v_{\sigma',k}$ if $\sigma'$ differs from $\sigma$
  in the $k$-th coordinate only. Writing the expression for $x_k$ 
  in (\ref{eq:gfvertices}) as $x_k = \pm z_k$, we thus get
\[
-z_k = \min(v_{\sigma,k},v_{\sigma',k}) <
\max(v_{\sigma,k},v_{\sigma',k}) = z_k,\] showing that $z_k>0$. It
follows that $\sign(v_{\sigma,k})=\sign(\sigma_k z_k) = \sign(\sigma_k)$.
\end{proof}

Now we are ready to state the crucial property of the Goldfarb cube
(which is invariant under translation and scaling, hence it applies to
our as well as the ``standard'' variant of the Goldfarb cube).

\begin{thm}[Theorem 4.4 in \cite{Amenta:1996uo}]\label{thm:shadow}
  Let $\pi:\R^d\rightarrow \R^2$ be the projection onto the last two
  coordinates, i.e.\
\[\pi((x_1,x_2,\ldots,x_{d-2},x_{d-1},x_d)^T) = (x_{d-1},x_d)^T.\]
The projection $\pi(\Gol_d) = \{\pi(\xx): \xx\in\Gol_d\}$ is a convex polygon
(two-dimensional polytope) with $2^d$ distinct vertices $\{\pi(\vv_{\sigma}):
\sigma\in\{-1,1\}^d\}$. In formulas, for every $\sigma\in\{-1,1\}^d$, there
exists an inequality $\aa^T\xx\leq 1$ such that $\aa\in \SS$ (recall that $\SS$ is the two-dimensional plane defined in (\ref{eq:S})) and 
\[
\begin{array}{lclclcl}
  \aa^T\vv_{\sigma} &=& a_{d-1}v_{\sigma,d-1} &+& a_d v_{\sigma,d} & = &1, \\
  \aa^T\xx &=& a_{d-1}x_{d-1} &+& a_dx_d &<& 1, \quad \xx\in \Gol_d\setminus\{\vv_{\sigma}\}.
\end{array}
\]
This precisely means that the inequality
\[a_{d-1}x+a_d y\leq 1\] defines the vertex
$\pi(\vv_{\sigma})=(v_{\sigma,d-1},v_{\sigma,d})^T$ of
$\pi(\Gol_d)=\{(x_{d-1},x_d)^T : \xx\in \Gol_d\}$.
\end{thm}
The set $\pi(\Gol_d)$ is the \emph{shadow} of $\Gol_d$ under the
projection $\pi$, and the theorem tells us that all Goldfarb cube
vertices appear on the boundary of the shadow. ``Usually'', the shadow 
of a polytope is of much smaller complexity, since many vertices 
project to its interior.

\subsection{Geometric Duality}\label{sec:duality}
There is a natural bijective transformation ${\cal D}$ that maps points $\pp=(p_1,\ldots,p_d)$ to
inequalities strictly satisfied by $\0$:
\[{\cal D}: (p_1,p_2,\ldots,p_d)^T \mapsto \{\xx\in\R^d: \pp^T\xx \leq
1\}.\] Using ${\cal D}$, we can map every set $\PP\subseteq\R^d$ to its
\emph{dual} (sometimes also called the \emph{polar} set)
\[\PP^{\triangle} := \bigcap_{\pp\in \PP}\{\xx\in\R^d: \pp^T\xx \leq 1\}.\]
If $\PP$ is a polytope with $\0\in\interior(\PP)$, given as 
the convex hull of a finite set of points $\VV$, then it can be shown that
\begin{equation}
\label{eq:polytopedual}
\PP^{\triangle} = \bigcap_{\vv\in \VV} \{\xx\in\R^d: \vv^T\xx \leq 1\}.
\end{equation}
This means, $\PP^{\triangle}$ is also a polytope, given as the solution
set of finitely many linear inequalities (boundedness follows from
$\0\in\interior(\PP)$).

This duality transform has two interesting properties that we need.
\begin{proposition}\label{propo:dual} 
Let $\PP\subseteq\R^d$ be a polytope containing the origin
in its interior, and let $\PP^{\triangle}$ be its dual polytope. 
\begin{itemize}
\item[(i)] $\PP=(\PP^{\triangle})^{\triangle}$, i.e.\ the dual of the dual
  is the original polytope.
\item[(ii)] If $\PP$ has $N$ vertices and $M$ facets, then
  $\PP^{\triangle}$ has $M$ vertices and $N$ facets. More precisely, $\vv$
  is a vertex of one of the polytopes if and only if the inequality
  $\vv^T\xx\leq 1$ defines a facet of the other.
\end{itemize}
\end{proposition}
As simple examples, we may consider the three-dimensional platonic
solids.  The geometric dual of a tetrahedron is again a tetrahedron. A
cube is dual to an octahedron, and a dodecahedron is dual to an
icosahedron.  The geometric dual of the $d$-dimensional unit cube is
the \emph{cross-polytope}, having $2d$ vertices and $2^d$ facets. The
dual of the Goldfarb cube is therefore a perturbed version of the
cross-polytope, see Figure~\ref{fig:crosspoly3}.

\subsection{The Dual Goldfarb Cube}

We are now able to follow up on our initial idea outlined in the
beginning of Section~\ref{sec:highdim}. By
Proposition~\ref{propo:dual}(ii), the dual Goldfarb cube
$\Gol^{\triangle}_d$ has $2d$ vertices and $2^d$ facets. Moreover, we
now easily see that all $2^d$ facets intersect the two-dimensional
plane $\SS$ defined in (\ref{eq:S}).  We in fact already know points
of $\SS$ in each of these facets.

\begin{cor}[of Theorem~\ref{thm:shadow}]\label{cor:shadow}
  Let $\sigma\in\{-1,1\}^d$. For the point $\aa=:\pp_{\sigma}\in\SS$
  as constructed in Theorem~\ref{thm:shadow}, we have
\begin{eqnarray}
  \pp_{\sigma}&\in& \Gol^{\triangle}_d\cap \SS,  \label{eq:p0}\\
  \pp_{\sigma}^T\vv_{\sigma} &=& 1, \label{eq:p3}\\
  \pp_{\sigma}^T\vv_{\tau} &<& 1, \quad \tau\neq\sigma.\label{eq:p4}
\end{eqnarray} 
\end{cor}
This means that $\pp_{\sigma}$ is in the \emph{$\sigma$-facet} of
$\Gol^{\triangle}_d$ defined by the inequality $\vv_{\sigma}^T\xx\leq
1$, but not in any other facet.

\begin{proof}
  Theorem~\ref{thm:shadow} readily guarantees $\pp_{\sigma}\in
  \SS$. Now we use the other two properties of $\pp_{\sigma}$ from the
  theorem:
\begin{eqnarray*}
  \pp_{\sigma}^T\vv_{\sigma} &=& 1,\\
  \pp_{\sigma}^T\xx &<& 1, \quad \xx \in \Gol_d\setminus\{\vv_{\sigma}\}.
\end{eqnarray*}
The first one is (\ref{eq:p3}), and using the second one with $\xx=\vv_{\tau}$ yields (\ref{eq:p4}).
Both properties together show that 
\[\pp_{\sigma}\in\Gol^{\triangle}_d=\bigcap_{\tau\in\{-1,1\}^d}\{\xx\in\R^d:\vv_{\tau}^T\xx\leq 1\},\]
where we are using (\ref{eq:polytopedual}) and
Proposition~\ref{propo:dual}(ii).
\end{proof}

We will need the following fact about the polygon
$\Gol^{\triangle}_d\cap\SS$.
\begin{lem}\label{lem:gfbounded}
Let $\xx\in\Gol^{\triangle}_d\cap\SS$. Then $x_{d-1}\leq 1$.
\end{lem}
\begin{proof}
By applying the definition of the dual polytope for the choice of two particular vertices $\vv_{\sigma}\in\Gol_d$ of the Goldfarb cube as defined in (\ref{eq:gfvertices}), we have that for all $\xx\in\Gol^{\triangle}_d$,
\begin{eqnarray*}
\vv_{(-1,\ldots,-1,1,-1)}^T\xx &=& 
(-1,\ldots,-1,1,-1+2\epsilon)^T\xx \leq 1, \\
\vv_{(-1,\ldots,-1,1,+1)}^T\xx&=& 
(-1,\ldots,-1,1,+1-2\epsilon)^T\xx \leq 1.
\end{eqnarray*}
Summing up both inequalities yields $(-2,\ldots,-2,2,0)^T\xx\leq 2$,
meaning that $x_{d-1}\leq 1$ if $\xx\in\SS$.
\end{proof}

We will also need the vertices of the dual Goldfarb cube. By geometric
duality, they are in one-to-one correspondence with the facets of
$\Gol_d$. Both can be indexed by the set
$\{1,\ldots,d\}\times\{-1,1\}$ as follows:

\begin{defn}\label{def:dualgfv}
  For $(k,s)\in\{1,\ldots,d\}\times\{-1,1\}$, let $\ww_{(k,s)}\in\R^d$
  be the unique vector such that for $s=-1$, the inequality $-z_k \leq
  x_k$ in (\ref{eq:goldfarb}) and for $s=1$ the inequality
  $x_k\leq z_k$ assumes the form
  \[\ww_{(k,s)}^T\xx\leq 1.\]
  According to Proposition~\ref{propo:dual}~(ii), the set 
  \[\{\ww_{(k,s)}: 1\leq k\leq d, ~s\in\{-1,1\}\}\]
  is exactly the set of the $2d$ vertices of the dual Goldfarb cube
  $\Gol^{\triangle}_d$.
\end{defn}

\begin{figure}[h]
\vskip -0.5em
\begin{center}
\centerline{\includegraphics[width=0.45\textwidth]{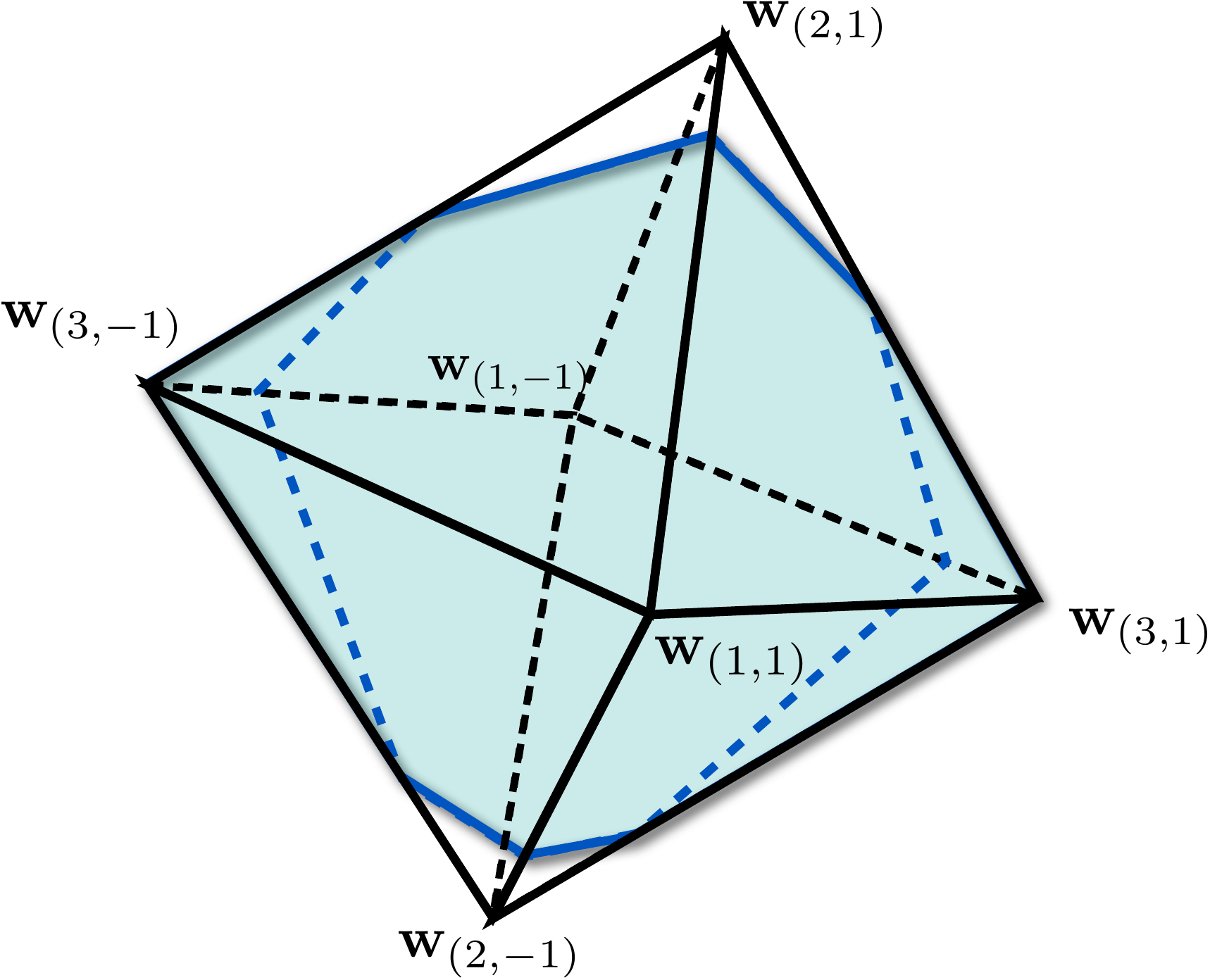}}
\vskip -0.8em
\caption{The dual of the Goldfarb cube in $3$ dimensions is the
  perturbed cross-polytope $\Gol^{\triangle}_3$. If we imagine the
  vertices $\ww_{(2,1)}$ and $\ww_{(2,-1)}$ lying just slightly behind the intersection
  plane $\SS$, and the vertices $\ww_{(3,1)}$ and $\ww_{(3,-1)}$ just slightly in front of
  $\SS$, then the plane~$\SS$ intersects all $2^3 = 8$ triangular
  facets.}
\label{fig:crosspoly3}
\end{center}
\vskip -0.5em
\end{figure} 

\subsection{Stretching}
Ideally, we would now like to use the vertices of the dual Goldfarb
cube $\Gol^{\triangle}_d$ as our first class of $n_+=2d$ points, and
make sure that the solution path ``walks along'' the exponentially
many facets that intersect the two-dimensional plane $\SS$ according to
Corollary~\ref{cor:shadow}. But for that, we need the walk to stay
close to $\SS$. To achieve this, we still need to ``stretch''
$\Gol^{\triangle}_d$ such that its facets are almost orthogonal to
$\SS$. The stretching transform scales all coordinates except the last
two by some fixed number $\Stretch$ (considered large).

\begin{defn}\label{defn:stretch}
For $\xx=(x_1,\ldots,x_{d})^T\in\R^d$
and $\Stretch\geq 0$ a real number, we define 
\[\xx(\Stretch) = (\Stretch x_1,\ldots,\Stretch x_{d-2},x_{d-1},x_d).\]
For a set $\PP\subseteq\R^d$, 
\[\PP(\Stretch) := \{\xx(\Stretch): \xx\in \PP\}\]
is the \emph{$\Stretch$-stretched} version of $\PP$. 
\end{defn}
The following is a straightforward consequence of this definition; we
omit the proof.
\begin{obs}\label{obs:stretching} 
  Let $\PP$ be a polytope and $\PP(\Stretch)$ its $\Stretch$-stretched
  version, $\Stretch\geq 0$.
\begin{itemize}
\item[(i)] $\PP\cap \SS = \PP(\Stretch)\cap \SS$, where $\SS$ is the 
two-dimensional plane defined in (\ref{eq:S}).
\item[(ii)] For $\Stretch>0$, the inequality $\aa^T\xx\leq 1$ defines
  the face $\FF$ of $\PP$ if and only if the inequality
  $\aa(1/\Stretch)^T\xx\leq 1$ defines the face $\FF(\Stretch)$ of
  $\PP(\Stretch)$.
\item[(iii)] For $\Stretch>0$, the point $\vv$ is a vertex of $\PP$ if and
  only if the point $\vv(\Stretch)$ is a vertex of $\PP(\Stretch)$.
\end{itemize}
\end{obs}

The idea behind the stretching transform is that for $\Stretch$ large
enough, the projection of any given point $\qq\in\SS$ onto
$\Gol^{\triangle}_d(\Stretch)$ is close to $\SS$. The following is the
key lemma; $\stretch$ assumes the role of $1/\Stretch$.

\begin{lem}\label{lem:convergence} 
  Let $\aa\in\R^d$ such that
  $(a_{d-1},a_d)\neq \0$. Fix a point $\qq\in \SS$ such that
  $\aa^T\qq>1$. For a real number $\stretch\geq 0$, let
  $\pp^{(\stretch)}$ be the projection (formally defined in the proof
  below) of $\qq$ onto the equality $\aa(\stretch)^T\xx = 1$. %
  Then \[\lim_{\stretch\rightarrow 0}\pp^{(\stretch)} = \pp^{(0)} \in \SS.\]
\end{lem} 

\begin{proof}
  The projection $\pp^{(\stretch)}$ can be defined through the
  equations
\begin{equation}\label{eq:projdef}
\aa(\stretch)^T\pp^{(\stretch)} = 1, \quad \pp^{(\stretch)}-\qq = t~\aa(\stretch) \mbox{~~for some $t$}.
\end{equation}
This is equivalent to
\begin{equation}
\label{eq:proj}
\pp^{(\stretch)} = C \frac{\aa(\stretch)}{\|\aa(\stretch)\|^2} + \qq, 
\mbox{~~with~~} C :=
1-\aa(\stretch)^T\qq = 1 - \aa^T\qq < 0.
\end{equation}
Now, since $\aa(\stretch)$ converges to $\aa(0)$ and
$\|\aa(\stretch)\|^2$ converges to $\|\aa(0)\|^2\neq 0$, the claim
follows; $\pp^{(0)} \in \SS$ is a consequence of $\qq,\aa(0)\in\SS$
and (\ref{eq:proj}).
\end{proof}

\subsection{Many Optimal Pairs}\label{sec:construction}
Let us now fix a sufficiently large stretch factor $\Stretch$ and
its inverse $\stretch=1/\Stretch$.  The goal of this section is to
construct a line $\LL\subseteq \SS$, disjoint from
$\Gol^{\triangle}_d(\Stretch)$, such that for exponentially many
$\sigma\in\{-1,1\}^d$, we find a pair of points $(\ppp,\qqq)$,
$\ppp\in\Gol^{\triangle}_d(\Stretch),\qqq\in \LL$, with the following properties.
\begin{itemize}
\item[(i)] $\ppp$ is in the $\sigma$-facet of the stretched dual
  Goldfarb cube, and in no other facet; and
\item[(ii)] $(\ppp,\qqq)$ is the unique pair of closest distance between
  the stretched dual Goldfarb cube and the ray $\{\xx\in \LL: x_d\geq
  q_{\sigma,d}\}$.
\end{itemize}

\begin{figure*}[ht]
\vskip 0.5em
\begin{center}
  \centerline{\includegraphics[width=0.75\textwidth]{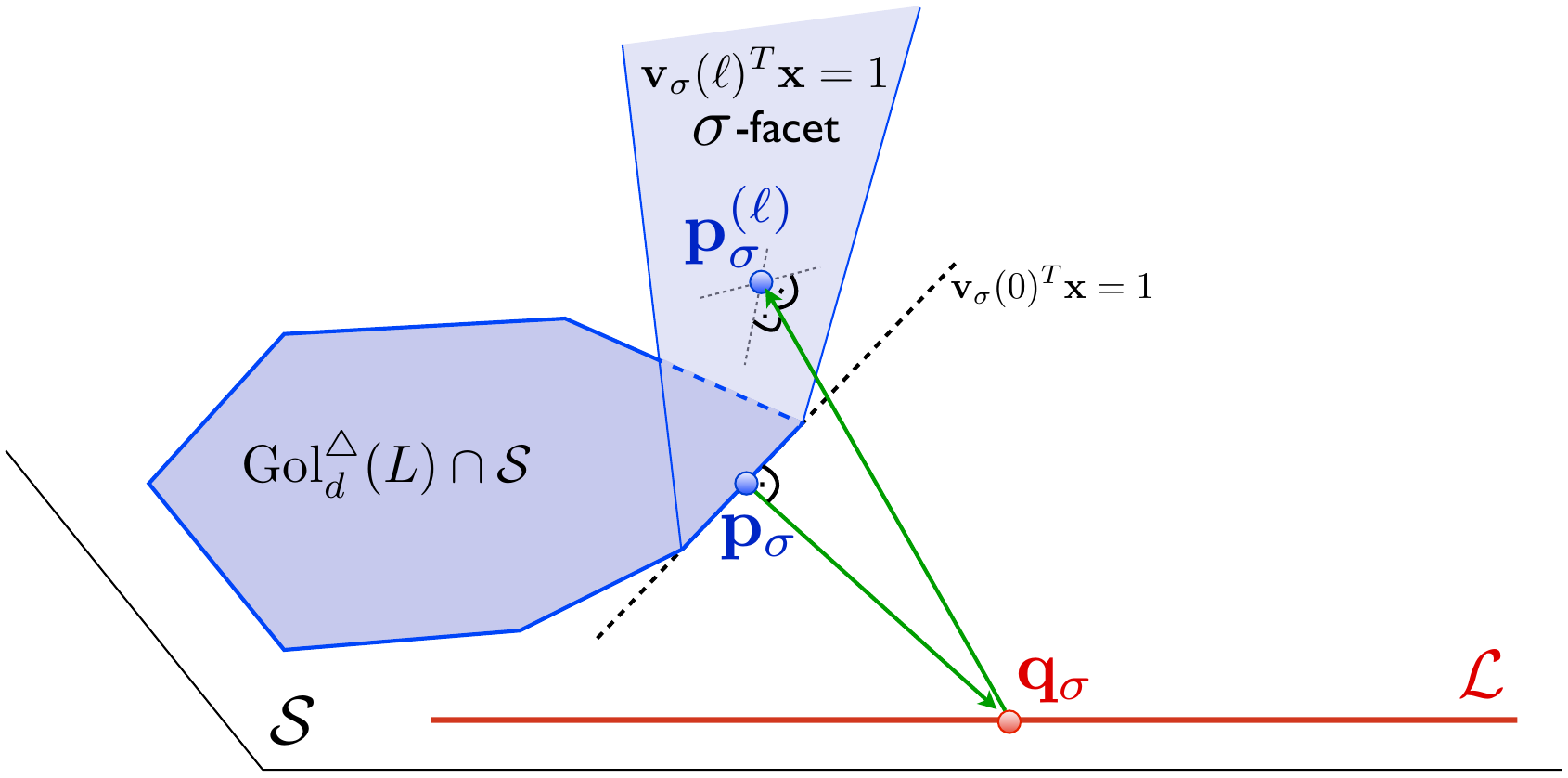}}
\caption{Obtaining the two points $\ppp$ and $\qqq$ by first ``projecting'' 
$\pp_{\sigma}$ onto the line $\LL$ and then back onto the $\sigma$-facet 
of the polytope $\Gol^{\triangle}_d(\Stretch)$.}
\label{fig:projectTwice}
\end{center}
\vskip -0.2em
\end{figure*}

\subsubsection{The Line}
The first step is to define the line $\LL$. We choose
\begin{equation}
\label{eq:L}
\LL := \{(0,\ldots,0,2,y)^T: y \in \R\} \subseteq \SS.
\end{equation}
This line is disjoint from $\Gol^{\triangle}_d(\Stretch)$ by
Lemma~\ref{lem:gfbounded}.

\subsubsection{The Point $\qqq$}
Let us now fix $\sigma\in\{-1,1\}^d$ such that $\sigma_{d-1}=1$.
According to Corollary~\ref{cor:gf12}, the Goldfarb cube vertex
$\vv_{\sigma}$ satisfies $v_{\sigma,d-1}>0$.

We start with the point $\pp_{\sigma}\in \Gol^{\triangle}_d\cap
\SS$ constructed in Corollary~\ref{cor:shadow}. This point is in the
$\sigma$-facet of $\Gol^{\triangle}_d$ defined by the inequality
$\vv_{\sigma}^T\xx\leq 1$. We next find a point $\qqq\in \LL$ such that
$\pp_{\sigma}$ is the projection of $\qqq$ onto the ``vertical''
inequality $\vv_{\sigma}(0)^T\xx\leq 1$.  See also Figure~\ref{fig:projectTwice}
for an illustration. According to (\ref{eq:proj}), $\qqq$ must satisfy
\begin{equation}\label{eq:pq}
\pp_{\sigma} = C \frac{\vv_{\sigma}(0)}{\|\vv_{\sigma}(0)\|^2} + \qqq, \quad
C = 1-\vv_{\sigma}(0)^T\qqq < 0.
\end{equation}
To get $\qqq$, we thus simply define
\begin{equation}\label{eq:q}
\qqq := \pp_{\sigma} - C\frac{\vv_{\sigma}(0)}{\|\vv_{\sigma}(0)\|^2} \in \SS,
\end{equation}
where $C$ is chosen such that $q_{\sigma,d-1}=2$. This is possible since
$v_{\sigma,d-1}\neq 0$. Premultiplying with $\vv_{\sigma}(0)^T$ shows
that
\[C=\underbrace{\vv_{\sigma}(0)^T\pp_{\sigma}}_{=\vv_{\sigma}^T\pp_{\sigma}=1}-\vv_{\sigma}(0)^T\qqq=
1-\vv_{\sigma}(0)^T\qqq,\] as required. Also, by using
Lemma~\ref{lem:gfbounded} and the defining equation~(\ref{eq:q}), 
we obtain that $C<0$, as a consequence of
\[q_{\sigma,d-1}=2=\underbrace{p_{\sigma,d-1}}_{\leq 1} - C\underbrace{v_{\sigma,d-1}}_{>0}.\]

\subsubsection{The Point $\ppp$}\label{sec:p}
With $\qqq$ as previously defined, we now define $\ppp$ by projecting
$\qqq$ back onto the $\sigma$-facet of our polytope, the stretched dual Goldfarb cube, see also
Figure~\ref{fig:projectTwice}.
Formally we set 
\begin{equation}\label{eq:pprime}
\ppp := C \frac{\vv_{\sigma}(\stretch)}{\|\vv_{\sigma}(\stretch)\|^2} + \qqq,
\quad C:=1-\vv_{\sigma}(\stretch)^T\qqq=1-\vv_{\sigma}(0)^T\qqq<0.
\end{equation}
By (\ref{eq:proj}), $\ppp$ is now the projection of $\qqq$ onto the
inequality $\vv_{\sigma}(\stretch)^T\xx\leq 1$ defining the
$\sigma$-facet of $\Gol^{\triangle}_d(\Stretch)$.

\subsubsection{Optimality of $(\ppp,\qqq)$}
For the pair $(\ppp,\qqq)$, items (i) and (ii) of the plan outlined in
the beginning of Section~\ref{sec:construction} remain to be proved.
We do this by the following main theorem, showing that the construction
works for $1/4$ of all choices of $\sigma$'s.

\begin{thm}\label{thm:optpq}
  For $\sigma\in\{-1,1\}^d$ such that $\sigma_{d-1}=\sigma_d=1$, let
  $\qqq$ and $\ppp$ be as defined in~(\ref{eq:q}) and
  (\ref{eq:pprime}). For sufficiently small $\stretch:=1/\Stretch>0$,
  the following two statements hold.
\begin{itemize}
\item[(i)] $\ppp\in \Gol^{\triangle}_d(\Stretch)$; in particular,  
\begin{eqnarray*}
\vv_{\sigma}(\stretch)^T\ppp &=& 1, \\
\vv_{\tau}(\stretch)^T\ppp &<& 1, \quad \tau\neq\sigma. 
\end{eqnarray*}
\item[(ii)] The pair $(\xx,\xx')=(\ppp,\qqq)$ is the unique optimal
  solution of the optimization problem
\begin{equation}\label{eq:zz}
\begin{array}{ll}
\minimize_{\xx,\xx'} & \|\xx-\xx'\| \\
\subjto & \xx\in\Gol^{\triangle}_d(\Stretch) \\
                  & \xx'\in\LL \\
                  & x'_d \geq q_{\sigma,d}.
\end{array}                 
\end{equation}
\end{itemize}
\end{thm}

\begin{proof}
  We have
  \[\ppp^T\vv_{\sigma}(\stretch) = 1\] by definition of $\ppp$, see
  (\ref{eq:projdef}).  As a consequence of (\ref{eq:p4}), the point
  $\pp_{\sigma}\in \SS$ satisfies
  \begin{equation}\label{eq:p4prime}
    \pp_{\sigma}^T\vv_{\tau}(0) = \pp_{\sigma}^T\vv_{\tau} < 1, \quad \tau\neq\sigma.
\end{equation}
Due to $\displaystyle\lim_{\stretch\rightarrow 0}\pp_{\sigma}^{(\stretch)} = \pp_{\sigma}$
(here we use $\pp_{\sigma}^{(0)}=\pp_{\sigma}$, see the ``Ansatz''
(\ref{eq:pq}), and Lemma~\ref{lem:convergence}), we also have
 \begin{equation}
   \lim_{\stretch\rightarrow 0}\ppp^T\vv_{\tau}(\stretch) =
   \pp_{\sigma}^T\vv_{\tau}(0) < 1,
\end{equation}
hence $\ppp^T\vv_{\tau}(\stretch)<1$ for sufficiently small
$\stretch$, and this proves part (i) of the theorem.

For the second part, we first observe that the problem (\ref{eq:zz})
can be written as a \emph{quadratic program}, the problem of
minimizing a convex quadratic function subject to linear (in)equality
constraints.  Indeed, after squaring the objective function, we obtain
the following equivalent program:
\begin{equation}\label{eq:zzprime}
\begin{array}{lrcl}
\minimize_{\xx,\xx'} & (\xx-\xx')^T(\xx-\xx') \\
\subjto & \vv_{\tau}(\stretch)^T\xx &\leq& 1, \quad \tau\in\{-1,1\}^d \\
                  & x'_i &=& 0, \quad i=1,\ldots,d-2 \\
                  & x'_{d-1} &=& 2\\ 
                  & x'_d &\geq& q_{\sigma,d}.
\end{array}                 
\end{equation}
For quadratic programs, the \emph{Karush-Kuhn-Tucker} optimality
conditions~\cite{Peressini:1988ug} are necessary and sufficient for the existence of an
optimal solution. Here, these conditions assume the following form: a
feasible solution $(\xx,\xx')$ of (\ref{eq:zzprime}) is optimal if and
only if there exist real numbers $\lambda_{\tau}\geq 0,\tau\in\{-1,1\}^d$ and
a vector $\Lam\in\R^d$, $\Lambda_d\leq 0$ such that
\begin{eqnarray}
  2(\xx-\xx') + 
  \sum_{\tau\in\{-1,1\}^d}\lambda_{\tau}\vv_{\tau}(\stretch) &=& 0 \label{eq:kkt1}\\
  2(\xx'-\xx) + \Lam &=& 0 \label{eq:kkt2}\\
  \lambda_{\tau}(\vv_{\tau}(\stretch)^T\xx-1) &=& 0, \quad \tau\in \{-1,1\}^d, \label{eq:kkt3} \\
  \Lambda_d(x'_d - q_{\sigma,d}) &=& 0. \label{eq:kkt4}
\end{eqnarray}
This easily yields that $(\xx,\xx')=(\ppp,\qqq)$ is indeed an optimal
pair.  According to (\ref{eq:pprime}), $\ppp-\qqq$ is a negative
multiple of $\vv_{\sigma}(\stretch)$, hence we may choose
$\lambda_{\sigma}>0$ and $\lambda_{\tau}=0,\tau\neq\sigma$ such that
(\ref{eq:kkt1}) is satisfied. To satisfy (\ref{eq:kkt2}), we simply
set $\Lam=2(\ppp-\qqq)$ and observe that indeed $\Lambda_{d}\leq 0$
since $\Lambda_d=p_{d}-q_{\sigma,d}$ is a negative multiple of
$v_{\sigma,d}(\stretch)=v_{\sigma,d}>0$ by our choice of $\sigma_d=1$
and Corollary~\ref{cor:gf12}.  The last two \emph{complementary
  slackness} conditions (\ref{eq:kkt3}) and (\ref{eq:kkt4}) are
satisfied due to $\vv_{\sigma}(\stretch)^T\ppp=1$ and $\xx'=\qqq$.

It remains to show that $(\ppp,\qqq)$ is the unique optimal pair. We
actually prove a stronger property: $(\ppp,\qqq)$ is the unique
optimal solution of the following relaxed problem, obtained after
dropping all inequalities $\vv_{\tau}(\stretch)^T\xx \leq 1$ for
$\tau\neq\sigma$.
\begin{equation}\label{eq:zzrelaxed}
\begin{array}{lrcl}
\minimize_{\xx,\xx'} & (\xx-\xx')^T(\xx-\xx') \\
\subjto & \vv_{\sigma}(\stretch)^T\xx &\leq& 1\\
                  & x'_i &=& 0, \quad i=1,\ldots,d-2 \\
                  & x'_{d-1} &=& 2\\ 
                  & x'_d &\geq& q_{\sigma,d}.
\end{array}                 
\end{equation}
First we prove that the relaxed problem has no other optimal solution
of the form $(\pp,\qqq)$. Due to $\vv_{\sigma}(\stretch)^T\qqq>1$, see
(\ref{eq:pprime}), we cannot have $\pp=\qqq$. Then, the
Karush-Kuhn-Tucker conditions
\begin{eqnarray*}
2(\xx-\xx') + 
\lambda_{\sigma}\vv_{\sigma}(\stretch)^T &=& 0, \quad \lambda_{\sigma}\geq 0 \\
2(\xx'-\xx) + \Lam &=& 0, \quad \Lambda_d \leq 0\\
\lambda_{\sigma}(\vv_{\sigma}(\stretch)^T\xx-1) &=& 0 \\
\Lambda_d(x'_d - q_{\sigma,d}) &=& 0
\end{eqnarray*}
for the relaxed problem require
$\pp-\qqq$ to be a strictly negative multiple of
$\vv_{\sigma}(\stretch)$.  Complementary slackness in turn implies
$\vv_{\sigma}(\stretch)^T\pp=1$, and according to (\ref{eq:pprime}),
this already determines $\pp=\ppp$, see the definition of projection
(\ref{eq:projdef}). To rule out an optimal solution $(\pp,\qq)$
with $\qq\neq\qqq$, we observe that $q_d > q_{\sigma,d}$ implies $\Lambda_d=0$ in
the Karush-Kuhn-Tucker conditions by complementary slackness. This
in turn yields $p_d=q_d$ and hence $\lambda_{\sigma}=0$ because
$v_{\sigma,d}(\stretch)>0$. But then $\pp=\qq$ which cannot be 
a solution because of
\[\vv_{\sigma}(\stretch)^T\qq = v_{\sigma,d-1}2 + \underbrace{v_{\sigma,d}}_{>0}q_d
\geq v_{\sigma,d-1}2 + v_{\sigma,d}q_{\sigma,d} = \vv_{\sigma}(\stretch)^T\qqq > 1.
\vspace{-1em}\]
\end{proof}

We still need to show that we have actually obtained ``many \emph{different}
optimal pairs''. But his is easy now.
 
\begin{cor}\label{cor:manydifferentpairs}
All points $\ppp$ considered in Theorem~\ref{thm:optpq} are pairwise
distinct, and so are all the points $\qqq$.
\end{cor}
\begin{proof}
Pairwise distinctness of the $\ppp$ immediately follows from statetment (i) of
Theorem~\ref{thm:optpq}. If we assume that $\qq_{\sigma}=\qq_{\sigma'}$ for
$\sigma\neq\sigma'$, then $(\ppp,\qqq)$ and 
$(\pp_{\sigma'}^{(\stretch)},\qq_{\sigma'})$ are distinct optimal pairs for
(\ref{eq:zz}) which contradicts statement (ii) of Theorem~\ref{thm:optpq}.
\end{proof}

\subsubsection{Constructing Support Vectors}
As we have outlined in the introductory Section~\ref{sec:geomsvm}, it is
standard that any solution to an SVM-like optimization problem can be expressed
in two ways: either as an explicit vector solving the \emph{primal} SVM
problem~(\ref{eq:svm}) or the distance version~(\ref{eq:redPolyDist}), or
secondly as a convex combination of the input points, if we consider the
corresponding \emph{dual} problem, which in our case
is~(\ref{eq:dsvm}). The input points appearing with non-zero coefficient in such
a convex combination are called the \emph{support vectors}.

For polytope distance problems, these two representations are even easier to see
and convert into each other, as a point is in a polytope if and only if it
is a convex combination of the vertices of the polytope, see also the polytope
basics in Section~\ref{sec:polbasics}.

We will now show that when using the stretched dual Goldfarb cube 
$\Gol^{\triangle}_d(\Stretch)$ as 
one point class of a polytope distance problem, then the
support vectors of the point $\ppp$ as constructed in
Section~\ref{sec:p} are precisely the $d$ vertices
$\ww_{(k,\sigma_k)}(\Stretch)$ of $\Gol^{\triangle}_d(\Stretch)$.
This means that for every chosen $\sigma\in\{-1,1\}^d$, we will get a different set
of support vectors for~$\ppp$. The following general lemma lets us
express a point $\pp\in\Gol^{\triangle}_d(\Stretch)$ as a unique
convex combination of its support vectors. Due to
Theorem~\ref{thm:optpq}, this lemma will in particular apply to our
solution points $\ppp$.

\begin{lem}\label{lem:support}
  Let $\sigma \in \{-1,1\}^d$, and $\pp\in\Gol^{\triangle}_d(\Stretch)$ such that
  \begin{eqnarray*}
  \vv_{\sigma}(\stretch)^T\pp &=& 1, \\
  \vv_{\tau}(\stretch)^T\pp &<& 1, \quad \tau\neq\sigma,
  \end{eqnarray*}
  where $\stretch=1/\Stretch$. Then we can write $\pp$ as a
  convex combination of exactly $d$ vertices, namely
  \begin{equation}\label{eq:pinconv2}
    \pp = \sum_{k=1}^d \alpha_{(k,\sigma_k)}\ww_{(k,\sigma_k)}(\Stretch), \quad 
    \sum_{k=1}^d\alpha_{(k,\sigma_k)} = 1, \quad
    \alpha_{(k,\sigma_k)}> 0~\forall k.
  \end{equation}
  Moreover, this convex combination is unique among all
  convex combinations of the $2d$ vertices $\ww_{(k,s)}(\Stretch)$,
  for $k \in \{1,\dots,d\}$ and $s\in\{-1,1\}$.
\end{lem}
\begin{proof}
  $\Gol^{\triangle}_d(\Stretch)$ is the convex hull of its
  $2d$ many vertices $\ww_{(k,s)}(L)$, see Section~\ref{sec:polbasics},
  Definition~\ref{def:dualgfv} and Observation~\ref{obs:stretching}.
  This means that $\pp$ can be written as some convex combination 
  of the form
  \begin{equation}\label{eq:pinconvgen}
  \pp = \sum_{(k,s)} \alpha_{(k,s)}\ww_{(k,s)}(\Stretch), \quad
  \sum_{(k,s)}\alpha_{(k,s)} = 1, \quad \alpha_{(k,s)}\geq 0~\forall
  (k,s),
  \end{equation}
  where $k \in \{1,\dots,d\}$ and $s\in\{-1,1\}$. Now Lemma~\ref{lem:F} 
  implies that all vertices
  $\ww_{(k,s)}(L)$ not on the $\sigma$-facet --- the ones for which
  \[\vv_{\sigma}(\stretch)^T\ww_{(k,s)}(L) =\vv_{\sigma}^T\ww_{(k,s)}<1\]
  must have coefficient $\alpha_{(k,s)}=0$. By
  Definition~\ref{def:dualgfv}, the inequalities $\ww_{(k,s)}^T\xx\leq
  1$ define the Goldfarb cube, and we know from
  Section~\ref{sec:goldfarb} that the vertex $\vv_{\sigma}$ is on
  \emph{exactly} the $d$ facets defined by the inequalities
  $\ww_{(k,\sigma_k)}^T\xx\leq 1$. Hence
  $\vv_{\sigma}^T\ww_{(k,-\sigma_k)}<1$, and
  $\alpha_{(k,-\sigma_k)}=0$ $\forall k$ follows. This means
  our convex combination is actually of the desired form~(\ref{eq:pinconv2})
  
  This also yields uniqueness of the $\alpha_{(k,s)}$: we know 
  from (\ref{eq:gfvertices}) %
  that the system of the $d$ equations
  \[\ww_{(k,\sigma_k)}^T\xx = 1, \mbox{~for~} 1\leq k\leq d\]
  uniquely determines $\vv_{\sigma}$, hence the $\ww_{(k,\sigma_k)}$
  and then also the $\ww_{(k,\sigma_k)}(\Stretch)$ are linearly
  independent. Therefore it follows
  that the convex combination~(\ref{eq:pinconvgen})
  must be unique (as we already know that all the $d$ 
  coefficients $\alpha_{(k,-\sigma_k)}$ must be zero anyway).\\

  It remains to show that $\alpha_{(k,\sigma_k)}>0$ $\forall k$. For this
  we suppose now that $\alpha_{(k,\sigma_k)}=0$ for some $k$. We obtain 
  $\sigma'$ from $\sigma$ by negating the $k$-th coordinate. We now have
  $\alpha_{(k,-\sigma'_k)}=0$ for all $k$, and by 
  applying the direction (i)$\Rightarrow$(ii) of Lemma~\ref{lem:F} 
  with ${\cal F}$ the $\sigma'$-facet of $\Gol^{\triangle}_d(\Stretch)$,
  we see that $\vv_{\sigma'}(\stretch)^T\pp=1$, a contradiction to our
  assumptions on $\pp$. So $\alpha_{(k,\sigma_k)}>0$ $\forall k$.
\end{proof}

A consequence of Lemma~\ref{lem:support} that we now see is that not
only $\ppp \in \conv(\PP)$, but also $\ppp \in \conv_{\mu}(\PP)$ for
$\mu$ sufficiently close to $1$. In the following, this will help us
to show that our constructed pairs of points are also optimal for a
distance problem between suitable reduced convex hulls. 

\begin{defn}\label{def:musigma}
  For $\sigma\in\{-1,1\}^d$, consider the unique positive coefficients
  $\alpha_{(k,\sigma_k)}$ obtained from Lemma~\ref{lem:support} for
  the point $\ppp$, and define
\[\mu^{(\stretch)}_{\sigma} := \max_{k=1}^d \,\alpha_{(k,\sigma_k)} < 1.\]
(If $d\geq 2$ positive coefficients sum up to $1$, their maximum must
be smaller than $1$).
\end{defn} 

\subsection{The Solution Path}
Let us summarize our findings so far: we have shown that there are
exponentially many distinct pairs $(\ppp,\qqq)$, each of them being
the unique pair of shortest distance between the 
stretched dual Goldfarb cube and the ray 
$\{\xx\in\LL:x_d\geq q_{\sigma,d}\}$, as shown by our optimality
Theorem~\ref{thm:optpq}. 

We still need to show that for suitable point classes, all these pairs
arise as solutions to the SVM distance problem~(\ref{eq:redPolyDist}),
for varying values of the parameter $\mu$. 

The first class of the SVM input points is given by the $n_+=2d$
vertices of the stretched dual Goldfarb cube $\Gol^{\triangle}_d(\Stretch)$, as
constructed in the previous Sections, or formally
\begin{equation}\label{eq:firstclass}
  \PP^+ := \SetOf{\ww_{(k,s)}(\Stretch)}{k \in \{1,\dots,d\}, s\in\{-1,1\}} ,
\end{equation}
so that $\conv(\PP^+) = \Gol^{\triangle}_d(\Stretch)$.
The second class of input points will be defined following the same idea as in
the first two-dimensional example given in Section~\ref{sec:2dimexmpl}: We
define it as just $n_-=2$ suitable points on the line $\LL$:
\begin{equation}\label{eq:secclass}
  \PP^- := \{\uu_{\lft},\uu_{\rgt}\} ,
\end{equation}
with
\begin{equation}
  \uu_{\lft} := \left(0,\ldots,0,2,u_{\lft,d}\right)^T~,~~
  \uu_{\rgt} := \left(0,\ldots,0,2,u_{\rgt,d}\right)^T .
\end{equation}
where suitable constants $u_{\lft,d} < u_{\rgt,d}$ will be fixed in the next 
section. The set $\PP^+ \cup \PP^-$ consisting of $n = n_+ + n_- = 2d + 2$ 
many input points is our constructed SVM instance.

Using these two point classes, we will now prove that as the regularization
parameter~$\mu$ changes, all our exponentially many constructed pairs
$(\ppp,\qqq)$ will indeed occur as optimal solutions on the solution path of the
SVM problem~(\ref{eq:redPolyDist}), and therefore also on the solution path of
the corresponding dual SVM~(\ref{eq:dsvm}).

Furthermore, we will also prove that we encounter exponentially many different
sets of support vectors (in the first point class) while the parameter
$\mu$ varies, by using the results of the previous section.

\subsubsection{Bringing in the Regularization Parameter}\label{sec:reducedStillFeasible}
In this section we will prove that for any chosen $\sigma$ with
$\sigma_{d-1}=\sigma_d=1$, our constructed pair of solution points $(\ppp,\qqq)$
will be the unique optimal solution to the SVM distance
problem~(\ref{eq:redPolyDist}) for some value of the parameter $\mu$.

So far, we have constructed support vectors w.r.t.\ the full convex hull
of the first point class $\PP^+$.
In the dual SVM formulation~(\ref{eq:dsvm}) and the distance
problem~(\ref{eq:redPolyDist}), this corresponds to the case $\mu=1$ or in other
words that 
the convex hulls are not reduced. In this small section we will prove 
that our constructed solutions and their corresponding support vectors 
of the first point class are actually valid for all $\mu$ sufficiently 
close to $1$, or formally that $\ppp \in \conv_{\mu}(\PP^+)$ for some $\mu < 1$.
This will enable us to transfer the optimality of our constructed pairs of
solution points $(\ppp,\qqq)$, as given by Theorem~\ref{thm:optpq}, also to the
distance problem~(\ref{eq:redPolyDist}), each pair being optimal for some unique
value of the parameter $\mu$.

\begin{defn} 
  Let $\overline{\mu} \in \R$ be the largest coefficient when writing all the
  $\ppp$ as their unique convex combination according to the
  ``support vector'' Lemma~\ref{lem:support}. Formally,
  \begin{equation}\label{eq:mumax}
  \overline{\mu} := \max\left\{ \frac{1}{2}\,,\, \max_{\sigma:\sigma_{d-1}=\sigma_d=1} 
    \mu^{(\stretch)}_{\sigma} \right\} < 1,
  \end{equation}
  see also Definition~\ref{def:musigma}. 
  Moreover, let $q_{\min}, q_{\max} \in \R$ be the smallest and largest ``horizontal position'' (or in other words last coordinate) of any
  of our constructed points $\qqq$, or formally
  \begin{equation}\label{eq:qminmax}
    q_{\min} := \min_{\sigma:\,\sigma_{d-1}=\sigma_d=1} ~ q_{\sigma,d}~, ~~~~~
    q_{\max} := \max_{\sigma:\,\sigma_{d-1}=\sigma_d=1} ~ q_{\sigma,d}.
  \end{equation}
\end{defn} 

Note that $\frac{1}{2} \leq \overline{\mu} < 1$ follows as 
the maximum is taken over $2^d/4$ many values which are all 
strictly smaller than $1$. Also, it must hold that
\begin{equation}\label{eq:qminmaxbounded}
-\infty < q_{\min} < q_{\max} < \infty.
\end{equation}
Here boundedness follows because also this minimum/maximum is over exactly 
$2^d/4$ many finite values, recall the definition of $\qqq$ in~(\ref{eq:q}) and the 
fact that $\|\vv_{\sigma}(0)\|^2 > 0$ $\forall \sigma$ (that follows from Corollary~\ref{cor:gf12}, applied
with $k=d-1,d$). Finally as the points $\qqq$ are distinct, as explained in 
Corollary~\ref{cor:manydifferentpairs}, we know that $q_{\min} < q_{\max}$.\\

Having computed $\overline{\mu}$ and the pair $q_{\min},q_{\max}$, we can now formally define 
the position of our two points $\uu_{\lft}, \uu_{\rgt}$ of the second point class. We choose 
their last coordinates as 

\begin{equation}\label{eq:urightconst}
u_{\lft,d} := q_{\min}~,~~~~
u_{\rgt,d} := q_{\min} + \frac{q_{\max} - q_{\min}}{1-\overline{\mu}}.
\end{equation}
The idea is that for this choice of the second class, and for a suitable value
of $\mu$ (depending on the point $q$) , the polytope $\conv_{\mu}(\PP^-)$ will be exactly
the first part of the ray $\SetOf{\xx\in\LL}{x_d \geq q_{d}} \subseteq
\LL$, as illustrated in Figure~\ref{fig:secondPointClass} and formally proved in 
the following lemma.

\begin{figure}[h]
\vskip -0.1em
\begin{center}
\centerline{\includegraphics[width=0.75\columnwidth]{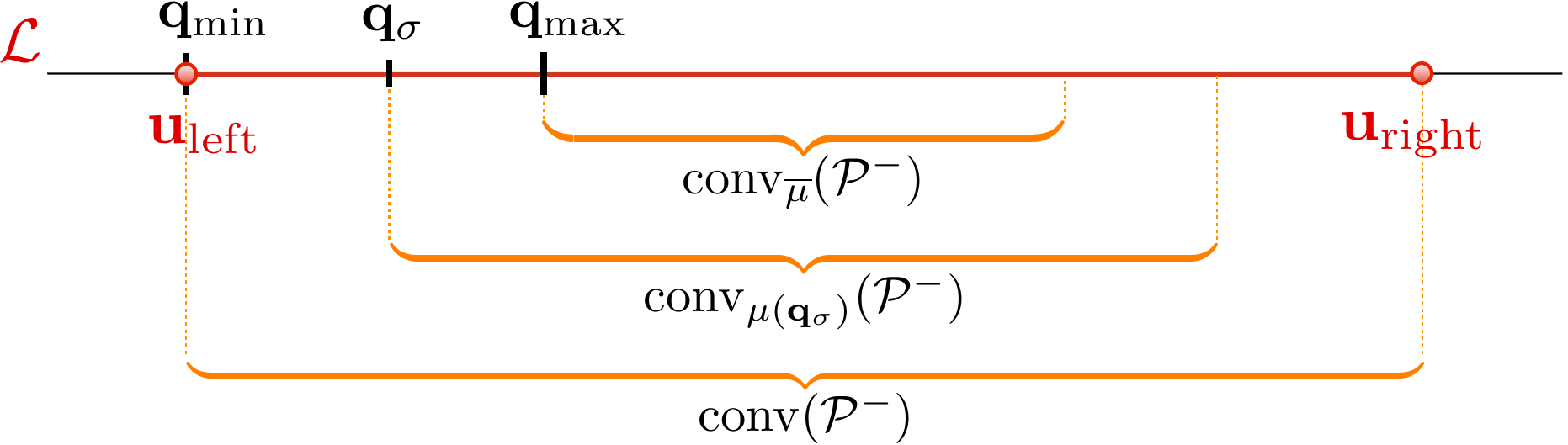}}
\caption{The second point class $\PP^- = \{\uu_{\lft}, \uu_{\rgt}\}$, arranged on the line $\LL$.
The reduced convex hulls are indicated for the three values $\overline{\mu} \leq \mu(\qqq) \leq 1$ of the regularization parameter~$\mu$.}
\label{fig:secondPointClass}
\end{center}
\vskip -0.5em
\end{figure}

\begin{lem}\label{lem:redsmallclass}
  Let $\qq$ be any point on the line $\LL$ satisfying $q_{\min} \leq q_{d} \leq q_{\max}$, 
  and define 
  \begin{equation}\label{eq:muq}
    \mu(\qq) := 1-\frac{(q_d-q_{\min})(1-\overline{\mu})}{q_{\max}-q_{\min}} ~.
  \end{equation}

  Then $\mu(q)\geq\overline{\mu}$, and the reduced convex hull of $\PP^-$ is exactly
  equal to the following non-empty line segment of $\LL$:
  $$
    \textstyle\conv_{\mu(\qq)}(\PP^-) 
    = \left[\qq,\uu_{\lft}+\uu_{\rgt}-\qq\right] 
    \subseteq \SetOf{\xx\in\LL}{x_d \geq q_{d}} .
  $$
\end{lem}
\begin{proof}
  For arbitrary two points $\PP^- = \{\uu_{\lft},\uu_{\rgt}\}$, it is easy to see that the
  reduced convex hull for any reduction factor $1\geq\mu\geq\frac{1}{2}$ is given by the line 
  segment $[\mu\uu_{\lft}+(1-\mu)\uu_{\rgt},\ \mu\uu_{\rgt}+(1-\mu)\uu_{\lft}]$. In
  our case, as $\uu_{\lft},\uu_{\rgt} \in \LL$, we are only interested in the $d$-th coordinate, and the 
  calculation is slightly simplified if we write $\lambda := \frac{1-\overline{\mu}}{q_{\max}-q_{\min}}$.
  We calculate the $d$-th coordinate of the left endpoint of the interval as
  $$
  \mu(\qq)u_{\lft,d}+(1-\mu(\qq)) u_{\rgt,d}
  = (1-(q_d-q_{\min})\lambda) q_{\min} + (q_d-q_{\min})\lambda \left(q_{\min} + \frac{1}{\lambda}\right)
  = q_d,
  $$
  and the right endpoint as
  \begin{eqnarray*}
  \mu(\qq) u_{\rgt,d} + (1-\mu(\qq))u_{\lft,d}
  &=& (1-(q_d-q_{\min})\lambda) \left(q_{\min} + \frac{1}{\lambda}\right) + (q_d-q_{\min})\lambda\,q_{\min}\\
  &=& q_{\min} + \frac{1}{\lambda} + q_{\min} - q_d = u_{\rgt,d} + u_{\lft,d} - q_d.
  \end{eqnarray*}
  This proves our claim that
  $$
    \textstyle\conv_{\mu(\qq)}(\PP^-) 
    = \left[\qq,\uu_{\lft}+\uu_{\rgt}-\qq\right] 
    \subseteq \SetOf{\xx\in\LL}{x_d \geq q_{d}},
  $$
  where inclusion in the line $\LL$ is clear as all points are part of $\LL$. 
  However it remains to show that this interval is non-empty and lies on the
  right-hand side of $q$, or formally that 
  $u_{\rgt,d} + u_{\lft,d} - q_d \geq q_d$.  Equivalently, the length of the interval is 
  $u_{\rgt,d} + u_{\lft,d} - q_d - q_d = \frac{q_{\max} - q_{\min}}{1-\overline{\mu}} 
  - 2 (q_d - q_{\min}) \geq 0$. Here the non-negativity follows from $1 > 
  \overline{\mu} \geq \frac{1}{2}$, so $\frac{1}{1-\overline{\mu}} \geq 2$, and 
  $q_d \leq q_{\max}$ by the definition of $q_{\max}$.
\end{proof}

\subsubsection{All Subsets of Support Vectors Do Appear Along the Path}

Note that for any $\sigma\in\{-1,1\}^d$ such that $\sigma_{d-1}=\sigma_d=1$, we 
have now computed a distinct regularization value $\mu(\qqq)$. We can now state 
the final theorem that for this parameter value, the same optimal solutions as in 
the optimality Theorem~\ref{thm:optpq} are also optimal for the SVM distance
problem~(\ref{eq:redPolyDist}), meaning that they realize the shortest distance between the two
reduced convex hulls $\conv_{\mu(\qqq)}(\PP^+)$ and $\conv_{\mu(\qqq)}(\PP^-)$:

\begin{thm}\label{thm:optprimalred}
  For every $\sigma\in\{-1,1\}^d$ such that $\sigma_{d-1}=\sigma_d=1$, 
  let $\qqq$ and $\ppp$ be as defined in (\ref{eq:q}) and
  (\ref{eq:pprime}). Then for sufficiently small $\stretch:=1/\Stretch>0$,
  the following two statements hold.
\begin{itemize}
\item[(i)] The pair $(\ppp,\qqq)$ is the unique optimal
  solution of the SVM optimization problem~(\ref{eq:redPolyDist}), which is
\begin{equation}\label{eq:redagain}
\begin{array}{ll}
\minimize_{\pp,\qq} & \|\pp-\qq\|^2 \\
\subjto & \pp\in\conv_{\mu(\qqq)} \left(\PP^+\right) \\
                  & \qq\in\conv_{\mu(\qqq)} \left(\PP^-\right).
\end{array}           
\end{equation}
\item[(ii)] When considering the optimal solution to the dual SVM 
  problem~(\ref{eq:dsvm}) for the regularization parameter value $\mu(\qqq)$, 
  the support vectors corresponding to the first point class~$\PP^+$ are 
  uniquely determined, and given by the $d$ vectors
$$
\SetOf{\ww_{(k,\sigma_k)}(\Stretch)}{k \in \{1,\dots,d\}} \ ,
$$
which is a different set for every single one of the $2^d/4$ many possible $\sigma$.
\end{itemize}
\end{thm}
\begin{proof}
  (i) By definition of the parameter $\mu(\qqq)$, we have that
  $$
  \textstyle\ppp \in \conv_{\mu(\qqq)}(\PP^+) \subseteq \conv(\PP^+) 
  = \Gol^{\triangle}_d(\Stretch)
  $$
  and from the previous Lemma~\ref{lem:redsmallclass} we know that
  $$
  \textstyle\qqq \in \conv_{\mu(\qqq)}(\PP^-) 
  = \left[\qqq,\uu_{\rgt}-\qqq\right]
  \subseteq \SetOf{\xx\in\LL}{x_d \geq q_{\sigma,d}}.
  $$
  In other words the two feasible sets $\conv_{\mu(\qqq)}(\PP^+)$, 
  $\conv_{\mu(\qqq)}(\PP^-)$ of the problem~(\ref{eq:redagain}) are 
  subsets of the feasible sets of the ``artificial'' distance problem~(\ref{eq:zz}),
  and the objective functions are the same.
  Also, we see that our pair of points $(\ppp,\qqq)$ is feasible for 
  both~(\ref{eq:zz}), but also the more restricted problem~(\ref{eq:redagain}).
  Therefore $(\ppp,\qqq)$ must be also optimal for the reduced hull 
  problem~(\ref{eq:redagain}), as Theorem~\ref{thm:optpq} tells us that
  it is already optimal for~(\ref{eq:zz}).
  
  For (ii), we apply the ``support vector'' Lemma~\ref{lem:support} for $\ppp$ to 
  get uniqueness. Optimality for~(\ref{eq:dsvm}) follows from the first part which 
  showed that $\ppp$ is optimal for the equivalent primal problem~(\ref{eq:redagain}).
\end{proof}

We have therefore established that exponentially many subsets of exactly 
$d$ support vectors out of $2d$ many input points occur as the regularization 
parameter $\mu$ changes between~$1$ and $\overline{\mu}$. The exact number of
distinct sets is $\frac{2^d}{4}$ when $d$ is the 
dimension of the space holding the input points, or $\frac{2^{n/2}}{8}$ if we express this complexity in the number of input points $n = n_+ + n_- = 2d + 2$. 

This also yields the same exponential lower bound for the number of bends in 
the solution path for $\mu \in [\overline{\mu},1]$, due to the following:
\begin{lem}
  Let $\ppp$ and $\pp_{\sigma'}^{(\stretch)}$ with $\sigma\neq\sigma'$
  be two points on the solution path (restricted to the first point
  class). Then the path has a bend between $\ppp$ and
  $\pp_{\sigma'}^{(\stretch)}$.
\end{lem}
\begin{proof}
  Suppose that the solution path includes the straight line
  segment connecting $\ppp$ and $\pp_{\sigma'}^{(\stretch)}$ (which
  are different by Corollary~\ref{cor:manydifferentpairs}). Let $\xx$
  be some point in the relative interior of that line segment. Then it
  follows from Theorem~\ref{thm:optpq}(i) that
  \[\vv_{\tau}(\stretch)^T\xx < 1\]
  for all $\tau$ which means that $\xx$ is not on the boundary of
  $\Gol^{\triangle}_d(\Stretch)$, a contradiction to $\xx$ being on the
  solution path.
\end{proof}

\section{Experiments}
We have implemented the above Goldfarb cube construction using
exact arithmetic, and could confirm the theoretical findings. We
constructed the stretched dual of the Goldfarb cube $\Gol_d$ using 
\texttt{Polymake} by \cite{Gawrilow:2005uu}. Figure~\ref{fig:intersection8example} 
shows the two dimensional intersection of the dual Goldfarb cube 
$\Gol^{\triangle}_d$ with the plane $\SS$.

\begin{figure}[h]
\vskip -0.1em
\begin{center}
\centerline{\includegraphics[width=0.4\textwidth,height=0.4\textwidth]{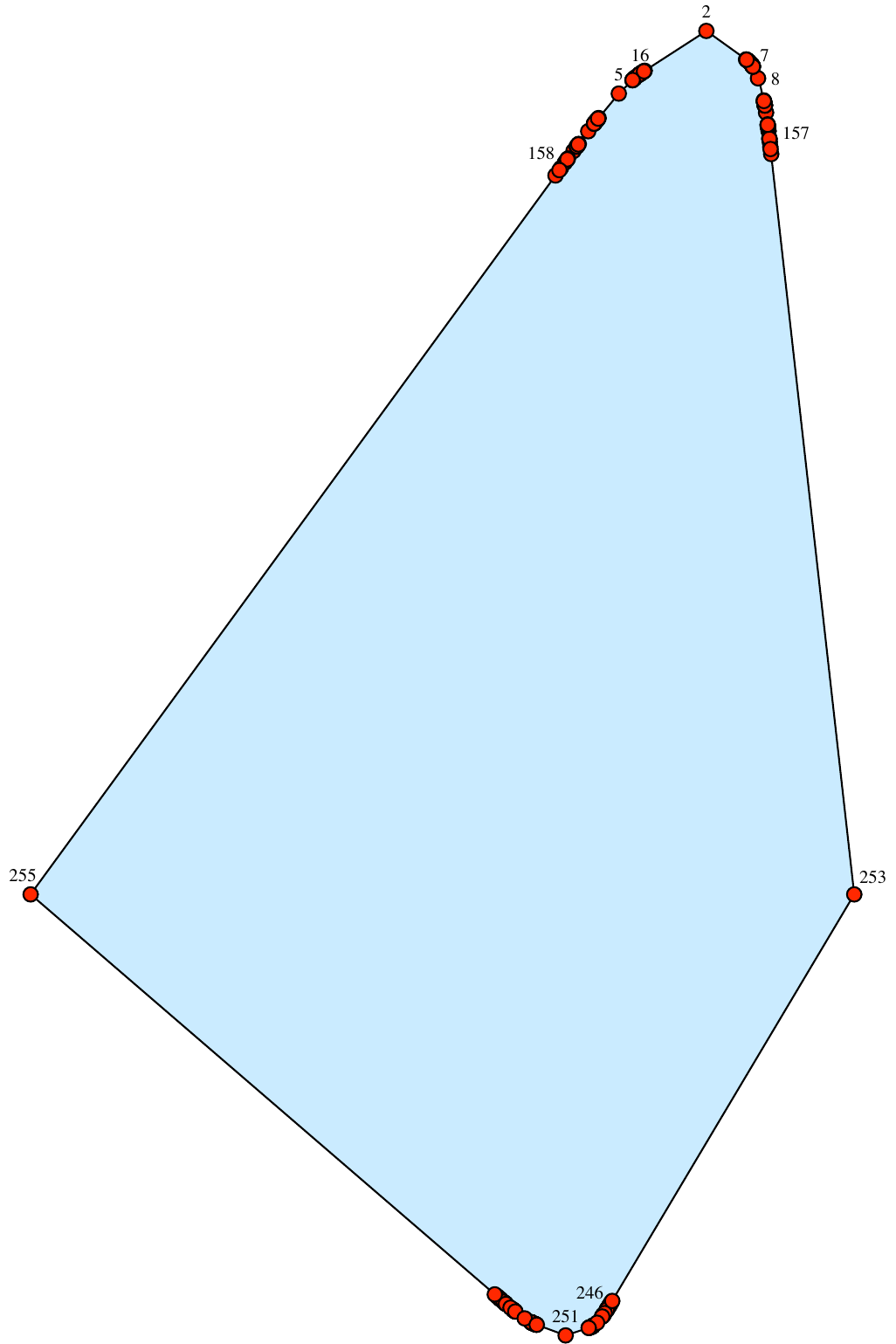}}
\vskip -0.3em
\caption{Example for $d=8$: The perturbed cross-polytope $\Gol^{\triangle}_8$ on 16 vertices
intersected with the two dimensional plane $\SS$ has 256 vertices.
Used command sequence in Polymake: 
\texttt{Goldfarb gfarb.poly 8 1/3 1/12;
center gcenter.poly gfarb.poly;
polarize gpolar.poly gcenter.poly;
intersection gint.poly gpolar.poly plane.poly;
polymake gint.poly}.
}
\label{fig:intersection8example}
\end{center}
\vskip -0.5em
\end{figure} 

Having obtained the vertices $\{\ww_{(k,s)}: 1\leq k\leq d, ~s\in\{-1,1\}\}$ of the polytope 
$\Gol^{\triangle}_d$ directly from \texttt{Polymake}, we then used the exact
(rational arithmetic) quadratic programming solver of \texttt{CGAL} \cite{cgal} to
calculate the optimal distance vectors between the polytopes 
$\conv_{\mu}(\PP^+) \subseteq \Gol^{\triangle}_d(\Stretch)$ and $\conv_{\mu}(\PP^-)$ for
some \emph{discrete} values of the parameter $\mu$. Here we just manually set 
the stretching factor as $L := 20'000$, and varied $\mu$ on a discrete grid within $[0.8,1]$.

For $d \le 8$, in all cases we obtained strictly more than our lower bound
of $\frac{2^d}{4} = \frac{1}{4} 2^{\frac{n_+}{2}}$ bends in the path. We
only counted a bend when the set of support vectors strictly changed when going
from one discrete $\mu$ value to the next. 

Note that it makes sense that the path complexity can be even higher than guaranteed by our lower bound from Theorem~\ref{thm:optprimalred}.
This is because in our construction, we have only considered the exponentially many \emph{original} facets of the point class $\conv(\PP^+)$, and none of the additional \emph{reduced} facets of the reduced convex hull $\conv_{\mu}(\PP^+)$ that occur when some of the coordinates $\alpha_p$ attain their upper bounds $\alpha_p \le \mu$ with equality, as the parameter $\mu$ becomes smaller.

\section{Conclusion}
We have shown that the worst case complexity of the solution path for SVMs 
--- as representing one type of parameterized quadratic programs --- is 
exponential both in the number of points $n$ and the dimension $d$. 
The example also shows that exponentially many (both in $n$ and $d$) 
distinct subsets of support vectors of the optimal solution occur as 
the regularization parameter changes.

We want to point out that our construction can also be interpreted as a 
general result in the theory of parameterized quadratic programs. Ignoring
the fact that we constructed an SVM instance, we have shown that the idea
of solving parameterized quadratic programs by tracking the solution path
leads to an exponential-time algorithm in the worst case.

Our result also implies that the complexity of the \emph{exact}
solution paths is quite different from the complexity of a path of
\emph{approximate} solutions (of some prescribed approximation
quality). For the SVM with $\ell_{2}$-loss, \cite{Giesen:2010fx,Jaggi:2011ux}
have shown that the complexity of such an approximate path is a
constant depending only on the approximation quality. It is thus
\emph{independent} of $n$ and $d$, for all inputs, which is in very
strong contrast to the worst-case complexity of the exact path  
as we proved here.

\paragraph{Acknowledgements.}
This project has been supported by the Swiss National Science Foundation (SNF Project 20PA21-121957).
Most of this work was done while M. Jaggi was at ETH Zurich, and C. Maria was visiting ETH Zurich.
We would like to thank the anonymous reviewers for helpful comments and suggestions.
We thank Joachim Giesen and Madhusudan Manjunath for stimulating discussions.

\newpage
\bibliographystyle{plain}
\bibliography{bibliography}

\end{document}